\documentclass{article}
\usepackage[letterpaper,margin=1in]{geometry}

\usepackage[utf8]{inputenc} % allow utf-8 input
\usepackage[T1]{fontenc}    % use 8-bit T1 fonts
\usepackage{xcolor}

\usepackage{url}
\usepackage{booktabs}       % professional-quality tables
\usepackage{amsfonts}       % blackboard math symbols
\usepackage{authblk}
\usepackage{nicefrac}       % compact symbols for 1/2, etc.
\usepackage{amsmath}
\usepackage{mathtools}
\usepackage{graphicx}
\usepackage{wrapfig}

\usepackage{amssymb}
\usepackage{caption}
\usepackage{subcaption}
\usepackage{multirow}
\usepackage[numbers]{natbib}
\usepackage{thmtools}
\usepackage{thm-restate}
\usepackage{amsthm}

\usepackage{algorithm}
\usepackage{algorithmic}

\title{Socratic Learning: Augmenting Generative Models to Incorporate Latent Subsets in Training Data}
\date{}

\author[ \hspace{-1ex}]{Paroma Varma}
\author[ \hspace{-1ex}]{Bryan He}
\author[ \hspace{-1ex}]{Dan Iter}
\author[ \hspace{-1ex}]{Peng Xu}
\author[*]{Rose Yu}
\author[$\dagger$]{Christopher De Sa}
\author[ \hspace{-1ex}]{Christopher R\'{e}}

\affil[ \hspace{-1ex}]{Stanford University, \textsuperscript{*}California Institute of Technology, \textsuperscript{$\dagger$}{Cornell University}
\affil[ \hspace{-1ex}]{\small \texttt{\{paroma,bryanhe,daniter,pengxu}@stanford.edu, qiyu@usc.edu\\cdesa@cs.cornell.edu, chrismre@cs.stanford.edu}}

\newcommand{\eat}[1]{}

\newcommand{\xs}{X_S}
\newcommand{\xsbar}{X_{\bar{S}}}

\begin{document}
\maketitle
	\begin{abstract}
		A challenge in training discriminative models like neural networks is obtaining enough labeled training data.
Recent approaches use generative models to combine weak supervision sources, like user-defined heuristics or knowledge bases, to label training data.
Prior work has explored learning accuracies for these sources even without ground truth labels, but they assume that a single accuracy parameter is sufficient to model the behavior of these sources over the entire training set.
In particular, they fail to model latent subsets in the training data in which the supervision sources perform differently than on average.
We present \emph{Socratic learning}, a paradigm that uses feedback from a corresponding discriminative model to automatically identify these subsets and augments the structure of the generative model accordingly.
Experimentally, we show that without any ground truth labels, the augmented generative model reduces error by up to $56.06\%$ for a relation extraction task compared to a state-of-the-art weak supervision technique that utilizes generative models. 

	\end{abstract}

	\section{Introduction}
	\label{sec:intro}
	Complex discriminative models like deep neural networks require significant amounts of training data.
For many real-world applications, obtaining this magnitude of labeled data is expensive and time-consuming. Recently, generative models have been used to combine noisy labels from various weak supervision sources, such as knowledge bases and user-defined heuristics \cite{xiao2015learning,ratner2016data,takamatsu2012reducing,roth2013combining,alfonseca2012pattern}.  They treat the true class label as a latent variable and learn a distribution over labels associated with each source using unlabeled data.
%They treat the true class label as a latent variable and learn the parameters associated with the sources using unlabeled data.
Probabilistic training labels, which can be used to train any discriminative model, are assigned to unlabeled data by sampling from this distribution. 

Even without ground truth data, generative models have been successful in learning accuracies for \cite{ratner2016data} and dependencies among \cite{bach2017learning} the different weak supervision sources. However, each of these methods relies on the crucial assumption that weak supervision sources have uniform accuracy over the \emph{entire} dataset. This assumption rarely holds for real-world tasks; for example, a knowledge base is usually relevant for part of the data being labeled and users develop their heuristics based on their understanding of the data, which is likely to be incomplete, and thus implicitly model only a subset of the data. As we show in Section~\ref{sec:exp}, failure to model these \emph{latent subsets} in training data can affect performance  by up to $3.33$ F1 points, measured in terms of the performance of the discriminative model trained on labels from the generative model.

%Intro Example
\begin{figure}[ht]
\centering
\centerline{\includegraphics[width=0.98\linewidth]{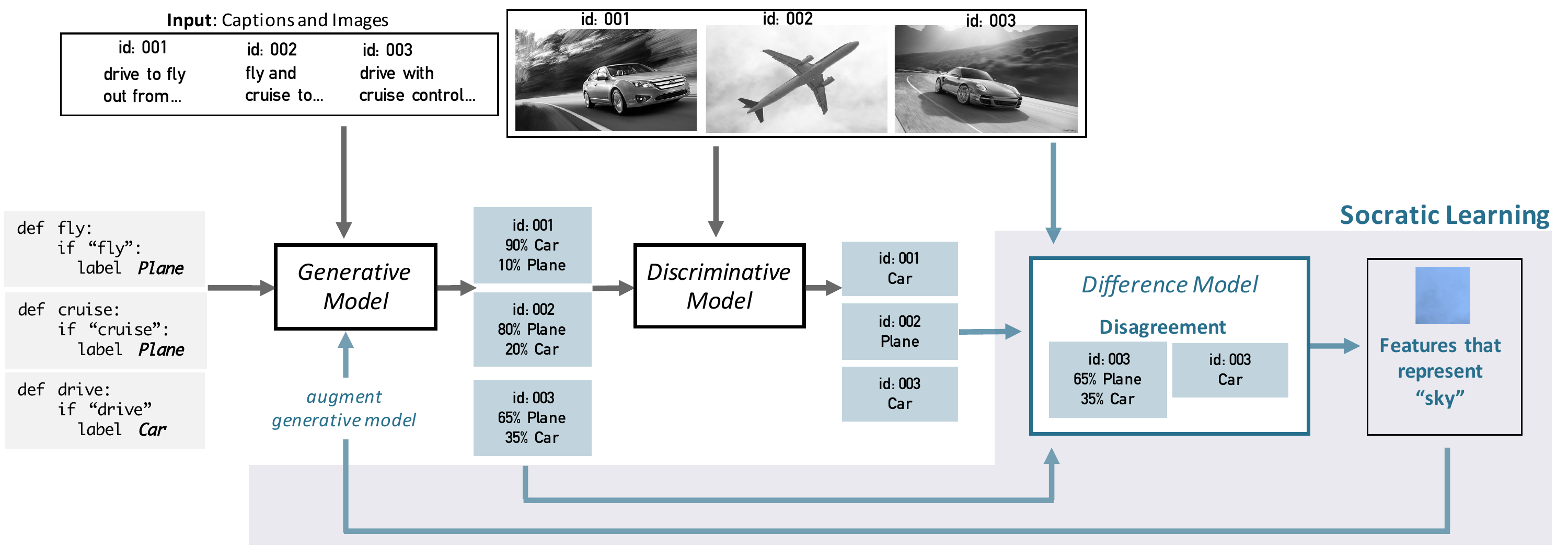}}
\caption{Toy example: The generative model receives the heuristics and captions as input. It creates a probabilistic training label for each object to pass to the discriminative model, which learns a representation over the images. The training labels and predictions are passed to the difference model, which selects features related to ``sky'' and passes it to the generative model.}
\label{fig:intro_eg}
\end{figure}

Accounting for latent subsets in training data when the true label is unknown is a challenging problem for several reasons. First, the lack of ground truth data makes it difficult for users to manually recognize subsets that their heuristics are more accurate for. Second, the generative model only has access to the observed variables, noisy labels from the weak supervision sources, which do not contain enough information to learn appropriate parameters to model the latent subsets. Finally, without knowledge of specific properties of the underlying data, it is difficult to identify the latent subsets that need to be modeled in order to accurately capture the underlying distribution of the training labels.
 
We present Socratic learning, a paradigm that uses information from the discriminative model to represent latent subsets in the training data and augments the generative model accordingly. To address the first challenge above, Socratic learning is designed to be an automated pipeline that requires \emph{no user input}. Next, to provide the generative model with knowledge about latent subsets in the training data, Socratic learning packages the necessary information using features from the corresponding discriminative model, which can be any model appropriate for the given problem. Using features of the data to represent the latent subsets ensures that the subsets are related to specific characteristics of the data and therefore likely to systematically affect accuracies of weak supervision sources. Finally, Socratic learning augments the generative model by adding extra parameters that capture different accuracies of the weak supervision sources conditional on the value of the representative features (Section~\ref{subsec: incorporating}).

While this process addresses the challenge of representing and incorporating latent subsets in the generative model, the challenge of identifying the features associated with the subsets without ground truth labels remains. Since, the corresponding discriminative model in our setting is more powerful than the generative model, it learns a better representation of the underlying data than the generative model. We therefore utilize this \emph{difference} in the generative and discriminative models of the data to identify which features are representative of the latent subsets. As shown in Figure~\ref{fig:intro_eg}, the difference model operates over the disagreement between generative and discriminative labels and identifies features that are most correlated with this disagreement. We prove that under a standard set of assumptions, the number of unlabeled data points needed to select the correct features with high probability scales logarithmically with the total number of features (Section~\ref{sec: difference}).

To describe Socratic learning concretely, we look at an image classification example shown in Figure~\ref{fig:intro_eg}. The generative model assigns probabilistic labels for the unlabeled captions after learning the accuracies of the heuristics \texttt{fly,cruise,plane}. These probabilistic labels and the associated images are used to train the corresponding discriminative model, which in turn predicts labels for the same data. Socratic learning uses the difference model to identify the disagreement between the two models' labels and selects the features most correlated with it. In this case, the features representative of the disagreement represent the presence of sky in the background. These are then passed to the generative model, which then learns a much higher accuracy for the \texttt{cruise} heuristic when these features are active.
%features representing the presence of sky in the background are passed to the generative model, which then learns a much higher accuracy for the \texttt{cruise} heuristic when these features are active. 

Our motivation for Socratic learning also comes from observing domain experts working within paradigms where the generative model is specified by user-defined heuristics \cite{ratner2016data}.
Heuristics for tasks like scientific text-relation extraction are complex rules that are, for example, pertinent only for specific types of documents can be difficult to debug manually.
Adding Socratic learning to their pipeline incorporates these nuances and automatically corrects the generative model. 
Biomedical experts confirm that the features that Socratic learning identifies are indeed correlated with the accuracy of the heuristics they wrote (Section~\ref{sec:exp}). 
We apply Socratic learning to a range of real-world datasets that utilize discriminative models such as CNNs~\cite{krizhevsky2012imagenet} and LSTMs~\cite{hochreiter1997long}.
Socratic learning improves discriminative model performance by reducing error in terms of achieving the performance of the fully supervised approach by up to $56.06\%$ for a text relation extraction task and $39.75\%$ for an image classification task.

\textbf{Summary of Contributions}
Our main contribution is the \emph{Socratic learning} paradigm, which addresses the issue of latent subsets in training data by augmenting the generative model defined by weak supervision sources.
Socratic learning relies on knowledge transfer between generative and discriminative models in the context of weak supervision to identify these subsets via the difference model. 
In Section~\ref{sec: difference}, we prove under standard assumptions that the difference model successfully selects features representative of the latent subsets. 
We then describe how we modify the single parameter version of the generative model to incorporate these features in Section~\ref{sec:method}. 
Finally, we report the performance of Socratic learning on various real-world datasets in Section~\ref{sec:exp}.

	\section{Background}
	\label{sec:bg}
	As the need for large, labeled training sets grows, recent methods use a generative model to combine various forms of weak supervision and assign noisy labels to available training data.
Socratic learning augments these generative models to model the latent subsets in training data.
To concretely describe the modifications made by Socratic learning on such generative models, we choose a recent approach, data programming \cite{ratner2016data}, which generalizes and incorporates various sources of weak supervision such as crowdsourcing, distant supervision, and hierarchical topic modeling.
Data programming models these sources as labeling functions, or heuristics, and uses the generative model to learn their accuracies and estimate probabilistic training labels, without the use of any ground-truth labels.

Data programming focuses on binary classification problems where each object $o \in O$ has an unknown true label $Y(o) \in \{-1,1\}$; a vector of features $V(o) \in \mathbb{R}^{Q}$; and a set of $M$ heuristics $\Lambda(o) \in \{-1,0,1\}^{M}$, which can vote for either class or abstain.
We have access to the features and heuristics for $N$ objects, but not their true label.
The goal is to learn a distribution for $Y(o)$ over all objects $o$.

\paragraph{Generative Model} 
The generative model $G$ uses a distribution in the following family to describe the relationship between the heuristics $\Lambda \in \mathbb{R}^{M \times N}$ and the latent variable $Y \in \mathbb{R}^{N}$, which represents the true class:
\begin{eqnarray}
	\pi_{\phi}(\Lambda, Y) = \frac{1}{Z_{\phi}} \exp\left(\phi^T \Lambda Y \right)
	\label{eqn:dp_factor_graph} 
\end{eqnarray}
where $Z_{\phi}$ is a partition function to ensure $\pi$ is a normalized distribution. The parameter $\phi \in \mathbb{R}^M$, which encodes the average accuracy of each of the $M$ heuristics, is estimated by maximizing the marginal likelihood of the observed heuristics $\Lambda$.
The generative model assigns probabilistic labels by computing $\pi_{\phi}(Y \mid \Lambda(o))$ for each object $o$ in the training set. 

\paragraph{Discriminative Model} The probabilistic labels for training data and the features $V \in \mathbb{R}^{N \times Q}$ are passed to the discriminative model $D$, which learns a classifier that can generalize over all objects, including those not labeled by the generative model.
Since the discriminative model has access to the probabilistic labels, it minimizes the noise-aware empirical risk of the following form for a logistic regression model:
\begin{eqnarray}
\label{eqn:noise_aware} 
  L_{\phi}(\theta) = \mathbb{E}_{Y \sim \pi_{\phi}(\cdot\mid\Lambda)}\left[\log\left(1+\exp(-\theta^TV^TY)\right)\right].
\end{eqnarray}
This loss takes the expectation over the distribution of the true class learned by the generative model. ,Any discriminative model can be made noise-aware by modifying its loss function similarly. % TODO: cut?
%\new{We will now refer to discriminative models with similarly modified loss functions as noise-aware discriminative models.}

	\section{Socratic Learning}
	\label{sec:method}
	In this section, we describe how Socratic learning initiates a cooperative dialog between the generative and discriminative models to incorporate information about latent subsets in the training data into the generative model.
First, we provide an overview of the Socratic learning algorithm.
We then show how to identify the latent subsets in the training data using the disagreement between the generative and discriminative models.
Finally, we describe how to incorporate the latent subsets into the generative model, resulting in a more expressive model that improves training set accuracy.

\subsection{Algorithm}
\label{subsec:alg}
% why do we need a difference model
Socratic learning first uses the original generative model, such as the one described in Equation~\eqref{eqn:dp_factor_graph}, to compute probabilistic labels $Y_G$ for the training data.
These labels are then used to train a discriminative model, which then computes its own labels $Y_D$ for the training data.
Socratic learning uses the disagreement $\tilde Y = -Y_GY_D$ between the labels assigned by the generative and discriminative models to identify the latent subsets that will help improve the generative model.
We describe how these features are selected in Section \ref{sec: difference}.
The generative model is then updated to model the latent subsets represented by the selected features.
These modifications are described in Section \ref{subsec: incorporating}.
The augmented generative model assigns new training labels that are more accurate than those from the original generative model.
These probabilistic labels are used to re-train the discriminative model.

In practice, the number of latent subsets is unknown.
As a result, the number of features to select ($K$) is treated as a hyperparameter and updated iteratively as described in Algorithm~\ref{alg:SL}.
The re-trained discriminative model is evaluated on a held-out development (dev) set at each step to track its performance.
In case no ground truth labels are available, the agreement between the generative and discriminative labels can also be used to decide when to stop passing features to the generative model, as shown in Figure~\ref{fig:disag}.
We observe that as $K$ increases, the quality of the generative model also increases, as witnessed by the increasing agreement between the generative and discriminative model labels. However, if too many features are passed to the generative model, then it can overfit, thus reducing agreement with the discriminative model.

%Algorithm
\begin{algorithm}[htb]
	\caption{Socratic  Learning}
		\label{alg:SL}
	\begin{algorithmic}
		\STATE {\bfseries Input}: Heuristics $\Lambda \in \{ -1,0,1\}^{M\times N}$, features $V \in \mathbb{R}^{N \times Q}$, features $X \in \{ -1,1\}^{N \times P}$
		\STATE Learn parameter $\phi$ for generative model $\pi_{\phi}(\Lambda, Y)$ (Equation~\eqref{eqn:dp_factor_graph})
		\STATE Assign probabilistic training labels $Y_G = [-1,1]^N$ with generative model
		\STATE Learn parameter $\theta$ that minimizes noise-aware discriminative loss $L_{\phi}(\theta)$ (Equation~\eqref{eqn:noise_aware})
		\STATE Predict labels $Y_D \in \lbrace -1,1 \rbrace ^ N$ with discriminative model
		\STATE Compute disagreement $\tilde{Y} = -Y_GY_D$
		\STATE $K=0$
		
		\REPEAT
		\STATE $K=K+1$
		\STATE Use difference model (with appropriate $\lambda$) to select features $X_{S_1\ldots S_K}$ indicative of $\tilde{Y}$ (Equation~\eqref{eq:lasso})
                \STATE Learn parameter $\phi$, $W$ for generative model $\pi_{\phi, W}(\Lambda, Y, X_{S_1\ldots S_K})$ and assign $Y_G$ (Section~\ref{subsec: incorporating}) 
                \STATE Learn parameter $\theta$ that minimizes noise-aware discriminative loss $L_{\phi}(\theta)$ and predict $Y_D$
		
		\UNTIL performance stops improving % (Section~\ref{subsec:alg})
	\end{algorithmic}
\end{algorithm}

%Disag Example Figure 
\begin{figure}[h]
  \centering
\centerline{\includegraphics[width=\linewidth]{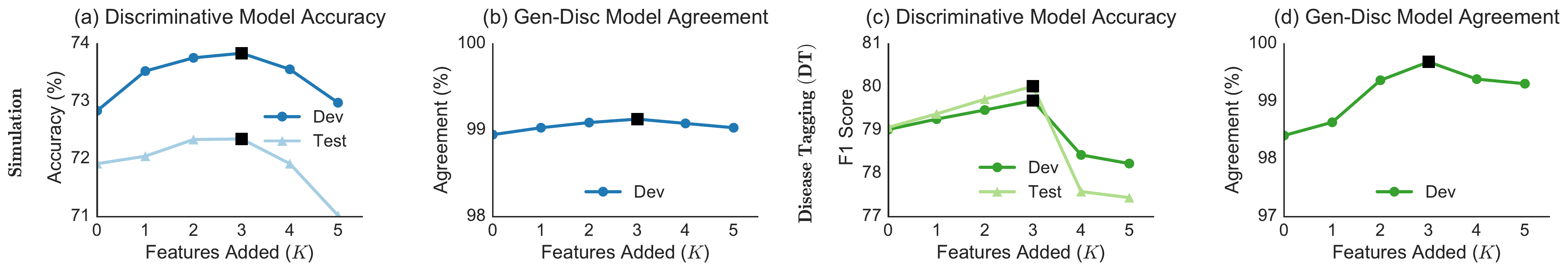}}
  \caption{(a),(c) Discriminative model improvement and (b),(d) agreement between generative and discriminative model labels vs. number of features added to the generative model ($K$) for simulation and DT experimental task.}
\label{fig:disag}
\end{figure} 

\subsection{Difference Model: Identifying Latent Classes}
\label{sec: difference}
We now describe the difference model used by Socratic learning to select the features, which are representative of the latent subsets in training data, to pass to the generative model.
The difference model takes in the probabilistic generative labels $Y_G\in [-1,1]^N$, the discriminative labels $Y_D\in\{-1,1\}^N$, and a set of binary features $X\in\{-1,1\}^{N \times P}$.
Note that these features may be different than the features $V\in\mathbb{R}^{N \times Q}$ used by the discriminative model, and we explore both scenarios in our experiments. % In our experiments, we explore both scenarios of $X$ being the same as and different than $V$.

Socratic learning recovers the features that are most useful for predicting the disagreement between the labels $\tilde Y = -Y_DY_G$ (element-wise product).
We denote these relevant features as the set $X_S$.
Since the set of relevant features is typically sparse, 
the difference model uses $\ell_1$-regularized linear regression (LASSO) and solves the following problem:
\begin{align}
  \label{eq:lasso}
  \hat{\theta} = \arg\min_{\theta} \|X\theta - \tilde Y\|_2^2 + \lambda \|\theta\|_1.
\end{align}
%where $\lambda$ is a regularization parameter used to control the number of features $K$ that are selected.
%The difference model then selects features corresponding to non-zero parameters that maximally correlate with the disagreement. These relevant features correspond to the set $X_S$ from Section 3.1.

Our setting differs from conventional LASSO problems in several ways.
First, in conventional LASSO problems, the data is assumed to be drawn from a noisy linear regression model.
In contrast, we only know that some features and the disagreement are correlated, but do not have any information about the parameters associated with these features. 
Second, in the conventional setting, all non-zero parameters should be recovered.
In the Socratic learning setting, we are not necessarily interested in recovering \emph{all} features correlated with the disagreement, only ones that will improve the generative model.
For example, we may not need a particular feature for predicting the disagreement if there is another feature very similar to it, since they would both relate to the same latent subset in the data.
As a result, our setting requires us to determine which features are \emph{necessary} to recover.

\paragraph{Theoretical Guarantees}
In the rest of this section, we provide sufficient conditions under which the correct set of features is recovered with high probability.
To specify our conditions, we will use the shorthand notations
\begin{align*}
\Sigma_{SS} = \frac{1}{N}\xs^T\xs \\
\Sigma_{S\bar S} = \frac{1}{N}\xs^T\xsbar
\end{align*}
where $\xsbar$ contains the features in $X$ but not in $\xs$.

%\todo{define alpha as a constant}
We now state two assumptions related to standard conditions used in the conventional setting.
First, we assume incoherence \cite{ravikumar2010high,wainwright2009sharp} %,fuchs2004recovery,tropp2006just
\begin{align}
\label{eq: incoherence}
\|\Sigma_{S\bar S}^T\Sigma_{SS}^{-1}\|_\infty \leq 1 - \alpha.
\end{align}
%which requires that the irrelevant features cannot exert an excessive effect on the relevant features.
Our next assumption resembles the dependence condition \cite{ravikumar2010high} and requires that the relevant features are not overly dependent on each other.
\begin{align}
\label{eq: dependence}
\|\Sigma_{SS}^{-1}\|_\infty \leq \beta.
\end{align}

We also describe properties of $\xs$ and $\xsbar$ that are necessary due to the fact that we only observe correlations, not the true parameters.
First, notice that $|\Sigma_{SS}^{-1}\mathbb{E}[\xs\tilde Y]|$ represents the predictive value of each element in $\xs$.
The relevant features must satisfy the following element-wise:
\begin{align}
  \label{prop: rel}
  |\Sigma_{SS}^{-1}\mathbb{E}[\xs\tilde Y]| \geq \gamma > 0,
\end{align}
which means that all relevant features have some predictive power for estimating the disagreement.

Next, for any irrelevant feature in $\xsbar$, we must have that
\begin{align}
  \label{prop: irrel}
  \left\|\mathbb{E}\left[\xsbar^T\left(\Pi_S - I\right)\tilde Y\right]\right\|_\infty \leq \frac{\alpha\gamma}{\beta} - c
\end{align}
for some $c>0$,
where $$\Pi_S = \xs\left(\xs^T\xs\right)^{-1} \xs^T$$ denotes the projection onto the space that can be predicted by the features $\xs$.
Notice that $(\Pi_S - I)\tilde Y$ then represents the residual of the disagreement that cannot be represented by the relevant features.
This property implies that none of the irrelevant features are strongly correlated with the residual after the relevant features are used. 

%Theorem
Under the stated conditions, we provide a guarantee on the probability of the difference model identifying the correct set of relevant features $\xs$ in Theorem~\ref{thmdiffmodel}. 

\begin{restatable}{theorem}{thmdiffmodel}
  \label{thmdiffmodel}
  Suppose we satisfy conditions \eqref{eq: incoherence} and \eqref{eq: dependence} and have relevant features $\xs$ satisfying property \eqref{prop: rel} and irrelevant features $\xsbar$ satisfying property \eqref{prop: irrel}.
  Then, for any tolerated failure probability $\delta > 0$,
  solving the LASSO problem \eqref{eq:lasso} with $\lambda=\frac{\gamma}{\beta}-\frac{c}{\alpha}$ and
  an unlabeled input dataset of size
\begin{align*}
    N
    \geq
    \frac{
    \max\left(
      8\beta^4 / \gamma^2,
    32
    \right)
}{c^2}
    \log\left(\frac{4P}{\delta}\right)
\end{align*}
    will succeed in recovering $\xs$ and $\xsbar$ with probability at least $1 - \delta$.
\end{restatable}

We now provide a proof sketch and include the full proof in the appendix.
First, consider the case in which there are enough samples to accurately estimate the correlations---that is, 
\begin{align*}
  &\frac{1}{N}\xs^T\tilde Y\approx\mathbb{E}[\xs\tilde Y] \\
  &\frac{1}{N}\xsbar^T\tilde Y\approx\mathbb{E}[\xsbar\tilde Y].
\end{align*}
In this case, we show that by solving problem \eqref{eq:lasso}, the correct features will always be selected.
We then consider the general case in which we only receive samples and have to estimate the correlations from the samples.
We use concentration inequalities to show that with high probability, we will fall in the regime of the previous case, implying that the correct features will be selected.

\begin{figure}[htbp]
  \centering
  \includegraphics[width=0.5\columnwidth]{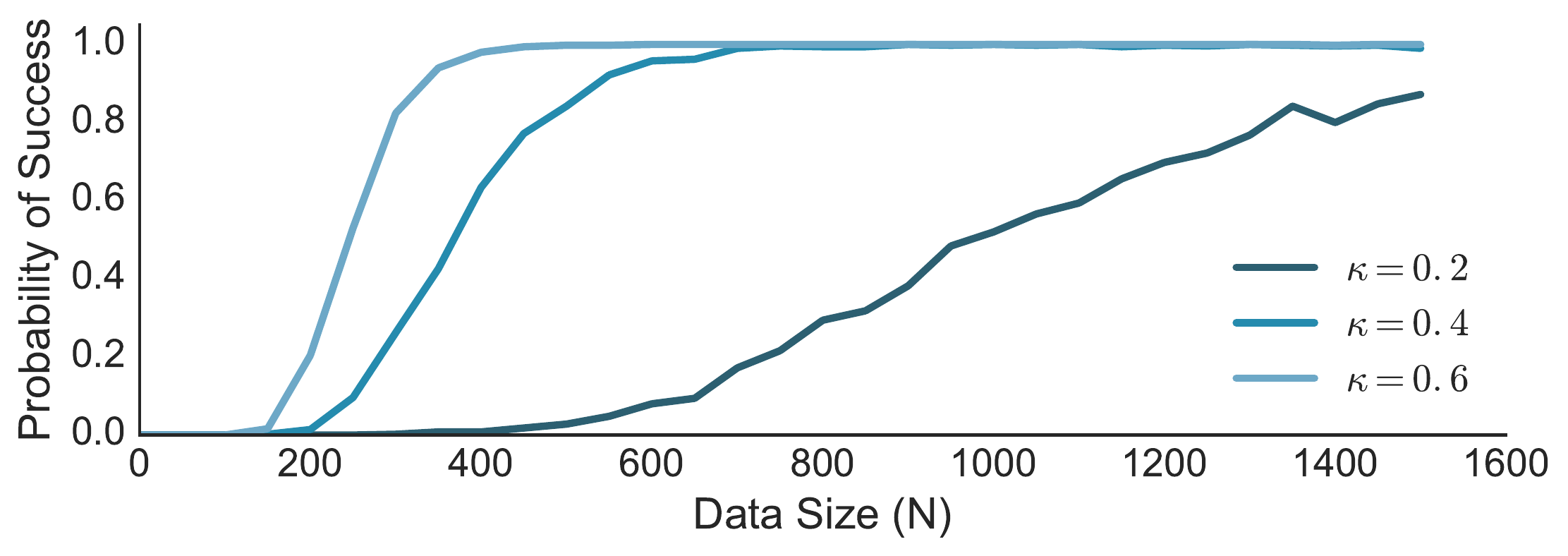}
  \caption{Probability of recovering relevant features $X_S$ with difference model with varying $\kappa = \mathbb{E}[X_S \tilde Y]$.}
  \label{fig:scaling}
\end{figure}

We run a simulation to verify that the correct features are recovered with high probability, results of which are shown in Figure \ref{fig:scaling}.
In this simulation, we have 100 features, of which 3 are relevant.
The expected correlation of the relevant features is selected from $0.2$, $0.4$, and $0.6$ in three experiments.
In each experiment, the probability of successfully recovering the 3 relevant features rapidly grows to 1 as the number of unlabeled data points is increased.
%In the next section, we apply the difference model to real-world settings, where we observe that the difference model continues to select features that improve the generative model.

\subsection{Incorporating Latent Classes into the Model}
\label{subsec: incorporating}
%NIPS New
Once the set of $K \ll P$ features representative of the latent subsets are selected, Socratic learning augments the generative model to account for them.
With access to the selected features, the generative model learns a different set of accuracies for each subset.
In the crowdsourcing setting, this could mean that for a subset of the data, we trust the workers more than on average since those data points are easier to label.
For knowledge bases, this could mean that a database is better suited for a portion of the data than the rest.

For concreteness, we describe this augmented generative model in Equation \eqref{eqn:dp_factor_graph}.
%Other generative models can be augmented in a similar manner.
%Returning to the generative model from Section~\ref{sec:bg}, we show how we modify Equation~\eqref{eqn:dp_factor_graph} to account for these latent classes in the data.
Socratic learning jointly models the relationship between the heuristics, the latent true class, and the features $X_S$ via the following family: 
\begin{align}
  \begin{split}
  &\pi_{\phi, W}(\Lambda, Y, X_S)=\\
  &\frac{1}{Z_{\phi, W}}
  %\Bigg[\exp \Big(\underbrace{\phi^T \Lambda Y\vphantom{\sum}}_{\mathclap{\textrm{Data Programming}}} +  \sum_{i=1}^K W_i^T \Lambda \mathrm{diag}(X_{S_i})Y \Big) \Bigg]
  \exp \Big(\phi^T \Lambda Y\vphantom{\sum} +  \sum_{i=1}^K W_i^T \Lambda \mathrm{diag}(X_{S_i})Y \Big)
  \end{split}
	\label{eqn:sl_factor_graph}
\end{align} 
where $Z_{\phi, W}\in\mathbb{R}$ is a partition function to normalize the distribution, and $W_i \in \mathbb{R}^M$ are weights associated with the features $X_{S_i} \in  \mathbb{R}^N$, $i=1,\ldots,K$.
The first term in the exponential models the heuristics as having a uniform accuracy over the whole data set.
The second term adjusts the accuracies by encoding the latent class in feature $X_{S_i}$. 
%This model has a larger number of parameters and thus captures a wider set of underlying distributions for the data than the model from Equation~\eqref{eqn:dp_factor_graph}. 
This model captures a wider set of underlying distributions for the data than the model from Equation~\eqref{eqn:dp_factor_graph}, while only requiring additional \emph{unlabeled} data to accurately learn the additional parameters, $W_i$. 
These parameters determine how the accuracy of the weak supervision sources are different for the latent subsets specified by selected features $\xs$.
This concept of modeling latent subsets in training data by identifying features that represent them, then learning a different set of parameters for them via the generative model can be extended to other approaches that use generative models to combine sources of weak supervision as well.

%Toy Example Figure
%\begin{figure}[ht]
%%\vskip -0.1in
%\begin{center}
%\centerline{\includegraphics[width=0.4\columnwidth]{images/misspecification_new-crop}}
%\caption{Example of a labeling function that has an accuracy of $75\%$ overall and different accuracies over two subsets in the data, correlated with feature $X_{S_1}$.}
%\label{fig:segment}
%\end{center}
%\vskip -0.2in
%\end{figure} 

	\section{Experimental Results}
	\label{sec:exp}
	% exp
%Stats Table
\begin{table}[!htbp]
\small
\centering
\caption{Weak Supervision (WS) Sources (KB: Knowledge Base, UDH: User-Defined Heuristic) and Statistics for Experimental Datasets (N: size of unlabeled training data, P: number of features passed to difference model, K: number of features passed to augmented generative model)}
\label{table:exp-stats}
\begin{tabular}{@{}cccccccccc@{}}
\toprule
\multirow{2}{*}{\textbf{Application}} & \multirow{2}{*}{\textbf{Domain}} & \multicolumn{3}{c}{\textbf{Source}}                  & \multicolumn{5}{c}{\textbf{Statistics}}                                                         \\ \cmidrule(l){3-5}  \cmidrule(l){6-10} 
                                      &                                  & \textbf{KBs} & \textbf{UDHs} & \textbf{Crowd Worker} & \textbf{\#WS} & \textbf{N} & \textbf{Labeled(\%)} & \textbf{P} & \multicolumn{1}{l}{\textbf{K}} \\ \midrule
CDR                                   & Text                             & \checkmark   & \checkmark    &                       & 33            & 8268       & 74.09                & 122,846    & 3                              \\
DT                                    & Text                             & \checkmark   & \checkmark    &                       & 16            & 11740      & 43.54                & 124,721    & 3                              \\
Twitter                               & Text                             &              &               & \checkmark            & 503           & 14551      & 89.41                & 118,379    & 2                              \\
MS-COCO                               & Text+Image                       &              & \checkmark    &                       & 3             & 8924       & 85.41                & 4096       & 2                              \\ \bottomrule
\end{tabular}
\end{table}

In this section, we explore classification tasks that range from image and sentiment classification to text relation extraction, with weak supervision sources that include knowledge bases, heuristics, and crowd workers. 
We compare the performance of discriminative models trained using labels generated by Socratic learning (SL), the single-parameter-per-source (SP) generative model, majority vote (MV), and the fully supervised (FS) case if ground truth is available.
We show that the performance of the \emph{end discriminative model} trained using labels from our generative model is better than that of the same model trained on labels from the SP model. 
We also discuss how we validated the features that were identified by the difference model by working with domain experts to interpret higher-level topics associated with them.

\subsection{Experimental Datasets}
We applied Socratic learning to two text-based relation extraction tasks, Disease Tagging (DT) and Chemical Disease Relation Extraction (CDR) \cite{cdr}, which look for mentions of diseases and valid disease-chemical pairs in PubMed abstracts. The CDR task relied on user-defined heuristics and distant supervision from the Comparative Toxicogenomics Databse (CTD) \cite{ctd} knowledge base. The DT task relied on user-defined heuristics and knowledge bases that were part of the Unified Medical Language System (UMLS) \cite{umls} collection.
For both tasks, domain experts fine-tuned their weak supervision sources over a period of several weeks; yet, SL is able to reduce error by up to $56.06\%$  compared to the SP generative model in terms of achieving the fully supervised score (Table~\ref{table:disc-stats}).

%Disc Model Stats Table
\begin{table}[!htbp]
\centering
\caption{Discriminative model performance in terms of improvement (Imp.) and error reduction using SL. SL consistently outperforms both MV and SP generative models, closing the gap between weakly and fully supervised models.}
\label{table:disc-stats}
\begin{tabular}{@{}cccccc@{}}
\toprule
\multirow{2}{*}{\textbf{Application}} & \multirow{2}{*}{\textbf{Model}} & \multicolumn{2}{c}{\textbf{Imp. Over}} & \multicolumn{2}{c}{\textbf{Error Reduction (\%)}} \\ \cmidrule(l){3-4} \cmidrule(l){5-6} 
                                      &                                 & \textbf{MV}           & \textbf{SP}           & \textbf{MV}             & \textbf{SP}             \\ \midrule
\multirow{2}{*}{CDR}                  & LR                              & 4.45                  & 3.33                  & 63.03                   & 56.06                   \\
                                      & LSTM                            & 2.70                  & 0.90                  & 42.19                   & 19.57                   \\
\multirow{2}{*}{DT}                   & LR                              & 1.74                  & 1.09                  & 20.49                   & 13.90                   \\
                                      & LSTM                            & 11.21                 & 0.99                  & 53.61                   & 9.26                    \\ \midrule
Twitter                               & LR                              & 4.07                  & 2.44                  & -                       & -                       \\
MS-COCO                               & AlexNet                         & 3.75                  & 2.23                  & 52.59                   & 39.75                   \\ \bottomrule
\end{tabular}
\end{table}

%\footnote{\url{https://www.crowdflower.com/data/airline-twitter-sentiment}}. \todo{cannot have footnotes, change to citation}
We also used crowd workers as sources of weak supervision for a Twitter sentiment analysis task \cite{twitter}, which meant the number of sources was significantly higher for this task than others.
Moving from the text to the image domain, we defined a task based on the MS-COCO dataset \cite{mscoco} where we wrote heuristic rules that used the content of captions to determine whether there was a human in the image or not. In each of these cases where the weak supervision sources were not tuned at all, SL is able to provide improvements of up to $2.44$ accuracy points, reducing error by up to  $ 39.75\%$. This demonstrates that even for relatively simple tasks or tasks with several sources of labels, SL can still correctly identify and augment the generative model and improve end discriminative model performance significantly.

The number of features passed to the generative model before the performance degrades is $3$ or less for each dataset. Although this is a seemingly small number of features, this leads to an exponential number of subsets in training data. Moreover, these subsets have to affect the training labels from the augmented generative model enough to lead to an improvement in the \emph{end discriminative model}. Finally, these features represent subsets that users did not account for while developing their weak supervision sources, despite domain experts spending weeks on the task in some cases.

\begin{figure}[!htbp]
  \centering
  \includegraphics[width=\linewidth]{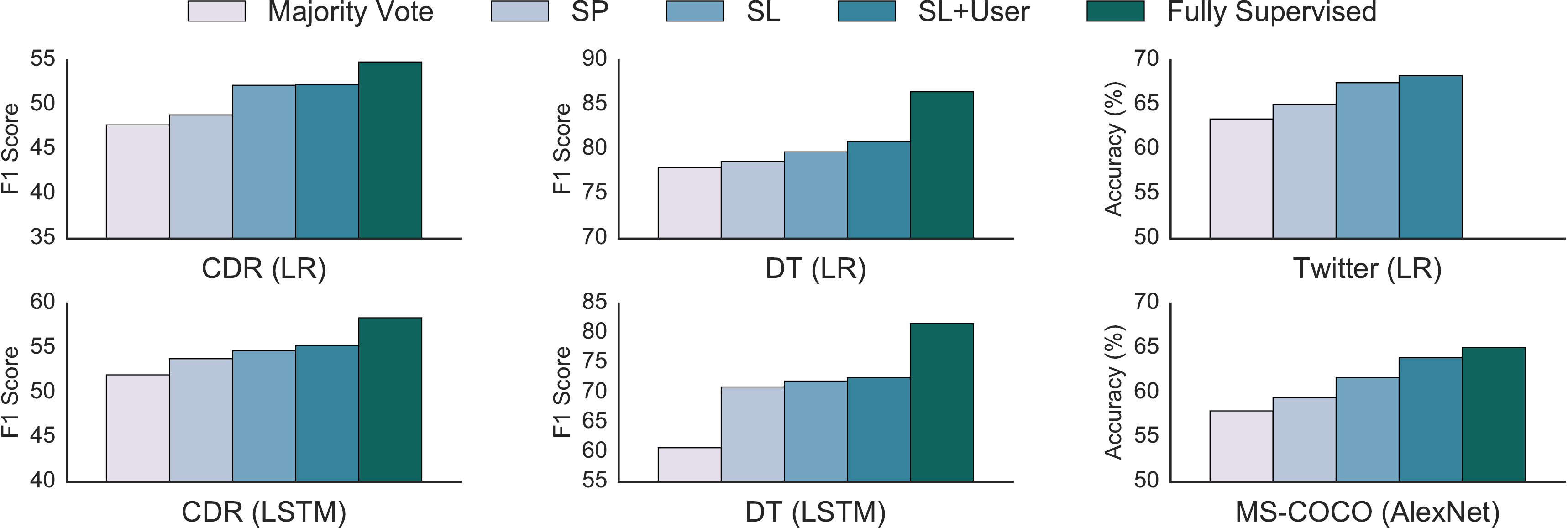}
  \caption{Discriminative model performance for different datasets and discriminative model choices comparing Socratic learning (SL) generative model, user-driven model augmentation in addition to Socratic learning (SL+U), single-parameter-per-source (SP) model, and majority vote. SL consistently outperforms SP and majority vote, while SL+U performs the best in each case. Absolute numbers provided in the Appendix.}
  \label{fig:results}
\end{figure}

\subsection{Performance of Discriminative Model} 
For the text relation extraction tasks, we used two discriminative models, logistic regression (LR) with hand-tuned features and LSTMs that take the raw text as input. The features for the discriminative and difference model are the same in the LR case and different for the LSTM case, in which the difference model features are the hand-tuned features and the discriminative model operates over the raw text.

We report the results in terms of F1 score, the harmonic mean of recall and precision. As shown in Table~\ref{table:disc-stats}, Socratic learning improves over the single parameter model consistently. With LR as the discriminative model, it improves by $3.33$ F1 points, and with the LSTM model, by $0.99$ F1 points, despite the weak supervision sources being extensively tuned by domain experts. As discussed in the next section, the latent subsets Socratic learning identifies are relevant to these highly technical tasks as well. 

SL performs even better on the Twitter and MS-COCO datasets, where the weak supervision sources are not finely tuned.  Even though the training labels were acquired from $503$ crowd workers for Twitter, SL provides an improvement of $4.07$ over majority vote and $2.44$ over SP (Table~\ref{table:disc-stats}), using a bag-of-words feature representation. For MS-COCO, we apply transfer learning by using AlexNet pre-trained on ImageNet, using the last layer as a feature extractor \cite{transferlearning}. Although the fairly simple task was to differentiate images that contained a person or not, SL is able to identify latent subsets in the training data that the heuristics did not model properly, reducing error by $39.75\%$ (Table~\ref{table:disc-stats}). This demonstrates that even seemingly simplistic datasets can have latent subsets that are difficult to recognize manually. 

\begin{table}[!htbp]
\centering
\caption{Discriminative model performance with user-driven model augmentation in addition to SL in terms of absolute improvement (Imp.) and error reduction. Users edit their weak supervision sources according to the their interpretation of features the difference model selects.}
\label{table:user-stats}
\begin{tabular}{@{}cccccc@{}}
\toprule
\multirow{2}{*}{\textbf{Application}} & \multirow{2}{*}{\textbf{Model}} & \multicolumn{2}{c}{\textbf{Imp. Over}} & \multicolumn{2}{c}{\textbf{Error Reduction (\%)}} \\ \cmidrule(l){3-4} \cmidrule(l){5-6} 
                                      &                                 & \textbf{SP}           & \textbf{SL}           & \textbf{SP}             & \textbf{SL}             \\ \midrule
\multirow{2}{*}{CDR}                  & LR                              & 3.43                  & 0.10                  & 57.74                   & 3.83                    \\
                                      & LSTM                            & 1.50                  & 0.60                  & 32.61                   & 16.22                   \\
\multirow{2}{*}{DT}                   & LR                              & 2.24                  & 1.15                  & 28.57                   & 17.04                   \\
                                      & LSTM                            & 1.59                  & 0.60                  & 14.87                   & 6.19                    \\ \midrule
Twitter                               & LR                              & 3.25                  & 0.81                  & -                       & -                       \\
MS-COCO                               & AlexNet                         & 4.47                  & 2.24                  & 79.68                   & 66.27                   \\ \bottomrule
\end{tabular}
\end{table}

\subsection{Features Selected by Difference Model}
For each of our experimental tasks, we validate the features that the difference model selects by showing users higher-level topics associated with them.
Note that while we describe how users interpret the features the difference model selects and manually improve their sources accordingly, this is beyond the automated improvement that Socratic learning provides without any user interaction. The core Socratic learning algorithm relies on automated generative model augmentation; the interpretability of the features is a side-effect that can lead to additional improvements, as described in this section.

For the DT task, one of the selected features represents ``induction of'', a common phrase that occurs with the word ``anesthesia'' when describing surgical procedures. 
One of the user-defined heuristics for this task marked candidate disease words as being false when the word ``anesthesia'' appeared nearby. 
Socratic learning correctly identified that this heuristic had a lower accuracy when the phrase ``induction of'' was also present, which biomedical experts confirmed is the case since surgical reports often contain names of potential diseases. 
For the CDR task, one of the features selected represented the word ``toxicity'' and affected the accuracy of heuristics that labeled a chemical-disease relation as true by looking for words like ``induce'' and ``associated''.
The presence of the word ``toxicity'' increased the accuracy of these sources since it suggested that the chemical was toxic, therefore increasing the likelihood that it induces a disease. 
While there were additional features that were detected for both the CDR and DT tasks, we omit them since they require significantly more exposition about the task and structure of PubMed abstracts. 
When biomedical experts edited their heuristics and the way they used knowledge bases given this information, they improved upon the the Socratic learning performance by up to $1.15$ F1 points  points (Table~\ref{table:user-stats}), demonstrating that the interpretability of features can help users manually improve the generative model as well.  
We provide details about the identified features for the remaining tasks in the Appendix.

	\section{Related Work}
	\label{sec:relate}
	% related work
Using weak supervision to efficiently label training data is a popular approach, relying on methods like distant supervision \cite{craven1999constructing,mintz2009distant}, multi-instance learning \cite{riedel2010modeling,hoffmann2011knowledge}, heuristic patterns \cite{hearst1992automatic,bunescu2007learning}, user input \cite{shin2015incremental,stewart2017label,ratner2016data}, and crowdsourcing \cite{dawid1979maximum}. 
Estimating the accuracies of these sources without access to ground truth labels is a classic problem \cite{dawid1979maximum}, especially in the field of crowdsourcing \cite{gao2011harnessing,dalvi2013aggregating,joglekar2015comprehensive,zhang2016spectral}. 
\citet{ruvolo2013exploiting} introduce the idea that these experts have specializations making their labels more accurate on a subset of the data. 
Socratic learning also finds latent subsets in training data, but does so without any ground truth data and for any type of weak supervision source, not only crowd workers. 

Recently, generative models have been used to combine various sources of weak supervision \cite{alfonseca2012pattern,takamatsu2012reducing,roth2013combining}.
These approaches can operate in the semi-supervised setting by modeling the ground-truth data as a very accurate source, but do not require any ground truth data.
\citet{xiao2015learning} proposes a generative approach to train CNNs with a few clean labels and numerous noisy labels by developing a probabilistic graphical model to describe the relationship between images and the labels. \citet{stewart2017label} show how to learn object detectors without any labels by incorporating hand-engineered constraint functions as part of the training objective.
\citet{reed2014training} avoids directly modeling the noise distribution via a bootstrapping approach by using a convex combination of noisy labels and the current model's predictions to generate the training targets. 

The specific example discussed in this paper, data programming \cite{ratner2016data}, proposes using multiple sources of weak supervision in order to describe a generative model and subsequently learn the accuracies of these source.
We chose data programming as a basis to explain Socratic learning since it subsumes various forms of weak supervision, making it a general framework to build upon. 
However, we note that Socratic learning can help identify latent subsets in training data with \emph{any} generative model that assigns training labels, since the difference model that identifies the subsets only needs labels from the generative and discriminative models and a set of features that represent the data. 
It can then augment the generative model for the specific method in a manner similar to the one described for data programming, where the model is forced to learn multiple parameters to represent the latent subsets in the data.

Finally, Socratic learning is also inspired by ideas similar to rule extraction \cite{thrun1995extracting,craven1997using}, where a complex model is simplified in terms of if-then statements. 
We instead simplify the connections between the latent subsets in training data and the weak supervision sources by augmenting the generative model in a manner such that there are implicit if-then statements that model the accuracy of each source conditioned on whether it is operating over a specific latent subset.

	\section{Conclusion}
	\label{sec:conc}
	% conclusion and future work
We introduced Socratic learning, a novel framework that initiates a cooperative dialog between the generative and  discriminative models. 
We demonstrated how the generative model can be improved using feedback from the discriminative model in the form of features. 
We reported how Socratic learning works with text relation extraction, crowdsourcing for sentiment analysis and multi-modal image classification, where it reduces error by up to $56.05\%$ for a relation extraction task and up to $39.75\%$ for an image classification task over a single parameter generative model.

\small
\paragraph*{Acknowledgments}
We thank Henry Ehrenberg and Alex Ratner for help with experiments, and Stephen Bach, Sen Wu, Jason Fries and Theodoros Rekatsinas for their helpful conversations and feedback.
We are grateful to Darvin Yi for his assistance with the DDSM dataset based experiments and associated deep learning models.
We acknowledge the use of the bone tumor dataset annotated by Drs. Christopher Beaulieu and Bao Do and carefully collected over his career by the late Henry H. Jones, M.D. (aka ``Bones Jones'').
This material is based on research sponsored by Defense Advanced Research Projects Agency (DARPA) under agreement number FA8750-17-2-0095.
We gratefully acknowledge the support of the DARPA SIMPLEX program under No. N66001-15-C-4043,
DARPA FA8750-12-2-0335 and FA8750-13-2-0039,
DOE 108845,
%personal
the National Science Foundation (NSF) Graduate Research Fellowship under No. DGE-114747,
Joseph W. and Hon Mai Goodman Stanford Graduate Fellowship,
Annenberg Graduate Fellowship,
%end personal
the Moore Foundation,
National Institute of Health (NIH) U54EB020405,
the Office of Naval Research (ONR) under awards No. N000141210041 and No. N000141310129,
the Moore Foundation,
the Okawa Research Grant,
American Family Insurance,
Accenture,
Toshiba, and Intel.
This research was supported in part by affiliate members and other supporters of the Stanford DAWN project: Intel, Microsoft, Teradata, and VMware.
The U.S. Government is authorized to reproduce and distribute reprints for Governmental purposes notwithstanding any copyright notation thereon.
The views and conclusions contained herein are those of the authors and should not be interpreted as necessarily representing the official policies or endorsements, either expressed or implied, of DARPA or the U.S. Government.
Any opinions, findings, and conclusions or recommendations expressed in this material are those of the authors and do not necessarily reflect the views of DARPA, AFRL, NSF, NIH, ONR, or the U.S. government.

\newpage
{
\bibliography{socrates_bib}{}
\bibliographystyle{abbrvnat}
}

\newpage
\appendix
\onecolumn
\section{Proof of Main Theorem}

%Lemma 3
\begin{restatable}{lemma}{lemmalasso}
  \label{lemmalasso}
  Suppose that we have
  \begin{gather}
    \|(\xsbar^T\xs)(\xs^T\xs)^{-1}\|_{\infty} \leq 1 - \alpha', \label{lemcond1} \\
    \|(\xs^T\xs)^{-1}\|_{\infty} \leq \beta', \label{lemcond2} \\
    |(\xs^T\xs)^{-1}\xs^T Y| \geq \gamma'\textrm{ element-wise}, \label{lemcond3} \\
    \|(\xsbar^T\xs)(\xs\xs)^{-1}\xs^T Y - \xsbar^T Y\|_{\infty} \leq \epsilon' \label{lemcond4} .
  \end{gather}
  Then, solving the problem
  \begin{gather*}
    \arg\min_{\theta} \frac{1}{2}\|X^T\theta - Y\|_2^2 + \lambda \|\theta\|_1
  \end{gather*}
  with the regularization parameter $\lambda$ bounded by
  \begin{align*}
    \frac{\epsilon'}{\alpha'} \leq \lambda <  \frac{\gamma'}{\beta'}
  \end{align*}
  results in all elements of $\theta_S$ being non-zero, and all elements of $\theta_{\bar S}$ being zero.
\end{restatable}
\begin{proof}
  To find the solution to the problem
  $$\arg\min_{\theta} \frac{1}{2}\|X\theta - Y\|_2^2 + \lambda \|\theta\|_1,$$
  we will set the subgradient to zero.

  We now find the zero subgradient condition.
  \begin{gather*}
    \nabla_\theta \left[\|X\theta - Y\|_2^2 + \lambda \|\theta\|_1\right] = \mathbf{0} \\
    \nabla_\theta \left[\theta^TX^TX\theta - 2\theta^TX^TY - Y^TY + \lambda \|\theta\|_1\right] = \mathbf{0} \\
    X^TX\theta - X^TY + \lambda \partial\|\theta\|_1 = \mathbf{0} \\
    X^TX\theta = X^TY - \lambda\partial\|\theta\|_1
  \end{gather*}
  Without loss of generality, we can say that all of the relevant features correspond to the top part of $X$ and $\theta$, and the irrelevant features correspond to the bottom part of $X$ and $\theta$.
  We can write the zero subgradient condition as
  \begin{gather*}
    \left[
      \begin{array}{cc}
        \xs^T\xs & \xs^T\xsbar \\
        \xsbar^T\xs & \xsbar^T\xsbar \\
      \end{array}
    \right]
    \left[
      \begin{array}{c}
        \theta_S\\\theta_{\bar S}
      \end{array}
    \right]
    =
    \left[
      \begin{array}{c}
        \xs Y\\\xsbar Y
      \end{array}
    \right]
    -
    \left[
      \begin{array}{c}
        \lambda \textrm{sign}(\theta_S) \\
        \lambda_{\bar S}
      \end{array}
    \right].
  \end{gather*}
  where $\|\lambda_{\bar S}\|_\infty \leq \lambda$.

  This problem is strictly convex, so any point where the subgradient is zero must be the unique minimizer.
  The remainder of this proof consists of two parts.
  First,  we show that there exists a solution where $\theta_{\bar S} = 0$ whenever $\lambda \geq \frac{\epsilon'}{\alpha'}$.
  Second, we show that there exists a solution where $\theta_{S} = 0$      whenever $\lambda <    \frac{\gamma'}{\beta'}$.

  First, we will prove that there exists a solution where $\theta_{\bar S} = 0$ when $\lambda \geq \frac{\epsilon'}{\alpha'}$.
  This implies that
  \begin{gather*}
    \lambda \geq \frac{\|\xsbar^T\xs\left(\xs^T\xs\right)^{-1} \xs^T Y - \xsbar^T Y\|_{\infty}}{1 - \|\xsbar^T\xs\left(\xs^T\xs\right)^{-1}\|_\infty} \\
    \|\xsbar^T\xs\left(\xs^T\xs\right)^{-1} \xs^T Y - \xsbar^T Y\|_{\infty} \leq \lambda \left(1 - \|\xsbar^T\xs\left(\xs^T\xs\right)^{-1}\|_\infty\right) \\
    \|\xsbar^T\xs\left(\xs^T\xs\right)^{-1} \xs^T Y - \xsbar^T Y\|_{\infty} \leq \lambda \left(1 - \|\xsbar^T\xs\left(\xs^T\xs\right)^{-1}\textrm{sign}(\theta_S)\|_\infty\right) \\
    \|\xsbar^T\xs\left(\xs^T\xs\right)^{-1} \xs^T Y - \xsbar^T Y\|_{\infty} + \|\lambda\xsbar^T\xs\left(\xs^T\xs\right)^{-1}\textrm{sign}(\theta_S)\|_\infty \leq \lambda \\
    \|\xsbar^T\xs\left(\xs^T\xs\right)^{-1} \xs^T Y  - \lambda\xsbar^T\xs\left(\xs^T\xs\right)^{-1}\textrm{sign}(\theta_S) - \xsbar^T Y\|_{\infty} \leq \lambda \\
    \|\xsbar^T\xs\theta_S - \xsbar^T Y\|_{\infty} \leq \lambda
  \end{gather*}
  This implies that the subgradient can be 0 when $\theta_{\bar S} = 0$.

  Next, we will prove that the subgradient is zero for some $\theta_S \neq 0$ with $\theta_{\bar S} = 0$ as long as $\lambda < \frac{\gamma'}{\beta'}$.
  This implies that
  \begin{gather*}
    \lambda < \frac{|(\xs^T\xs)^{-1}\xs^T Y|}{\|(\xs^T\xs)^{-1}\|_{\infty}} \\
    \lambda\|(\xs^T\xs)^{-1}\|_{\infty} < |(\xs^T\xs)^{-1}\xs^T Y| \\
    \lambda\|(\xs^T\xs)^{-1}\textrm{sign}(\theta_S)\|_{\infty} < |(\xs^T\xs)^{-1}\xs^T Y| \\
     |(\xs^T\xs)^{-1}\xs^T Y - \lambda(\xs^T\xs)^{-1}\textrm{sign}(\theta_S)| \neq 0
  \end{gather*}
  Then, for the subgradient to be zero, it must be the case that $\theta_S\neq 0$.

\end{proof}

%\subsection{Proof of Main Theorem}
\thmdiffmodel*

\begin{proof}
  To prove this theorem, we will show that we satisfy the requirements of Lemma~\ref{lemmalasso} with high probability.
  In particular, we need conditions \eqref{lemcond1}, \eqref{lemcond2}, \eqref{lemcond3}, and \eqref{lemcond4} to hold for parameters where
  \begin{align*}
    \frac{\epsilon'}{\alpha'} < \frac{\gamma'}{\beta'}.
  \end{align*}

  First, notice that condition \eqref{eq: incoherence} implies that condition \eqref{lemcond1} holds for $\alpha' = \alpha$, and
                     condition \eqref{eq: dependence}  implies that condition \eqref{lemcond2} holds for $\beta' = \beta/N$.
  Now, it remains to show that conditions \eqref{lemcond1} and \eqref{lemcond2} hold for $\gamma'$ and $\epsilon'$ such that the desired inequality holds with high probability.

  Property \eqref{prop: rel} implies that the expected value of $|(\xs^T\xs)^{-1}\xs^T\tilde Y|$ is at least $\gamma$ for each element, and property \eqref{prop: irrel} implies that the expected value of $\|(\xsbar^T\xs)(\xs\xs)^{-1}\xs\tilde Y - \xsbar \tilde Y\|_{\infty}$ is at most $N\left(\frac{\alpha\gamma}{\beta} - c\right)$.
  %\begin{align*}
  %  \epsilon < \frac{\alpha\gamma}{\beta} - c
  %\end{align*}
  %to hold.
  Now, it remains to show that the original inequality will hold with high probability with finite samples.

  To ensure that the original inequality holds,
  if suffices to ensure that
  no term of $|(\xs^T\xs)^{-1}\xs\tilde Y|$ deviates from the mean by more than $\frac{c\alpha}{2\beta}$, and that
  no term of $\frac{1}{N}\|(\xsbar^T\xs)(\xs\xs)^{-1}\xs\tilde Y - \xsbar \tilde Y\|_{\infty}$ deviates from the mean by more than $\frac{c}{2}$.

  First, we will bound the probability that any term of $|(\xs^T\xs)^{-1}\xs\tilde Y|$ exceeds the mean by more than $\frac{c\alpha}{2\beta}$.
  Each term is the sum of $N$ terms bounded in absolute value by $\beta/N$.
  Then, by the Asuma-Hoeffding inequality, the probability that a particular term deviates from the mean by more than $\frac{c\alpha}{2\beta}$ is at most
  \begin{align*}
    2\exp\left(-\frac{Nc^2\alpha^2}{8\beta^4}\right).
  \end{align*}
  By the union bound, the probability that any term deviates from the mean by more than $\frac{c\alpha}{2\beta}$ is at most
  \begin{align*}
    2P\exp\left(-\frac{Nc^2\alpha^2}{8\beta^4}\right).
  \end{align*}

  Next, we will bound the probability that any term of $\frac{1}{N}\|(\xsbar^T\xs)(\xs\xs)^{-1}\xs\tilde Y - \xsbar \tilde Y\|_{\infty}$ exceeds the mean by more than $\frac{c}{2}$.
  Each term is the sum of $N$ terms bounded in absolute value by $\frac{2}{N}$.
  Then, by the Asuma-Hoeffding inequality, the probability that a particular term deviates from the mean by more than $\frac{c}{2}$ is at most
  \begin{align*}
    2\exp\left(-\frac{Nc^2}{32}\right).
  \end{align*}
  By the union bound, the probability that any term deviates from the mean by more than $\frac{c}{2}$ is at most
  \begin{align*}
    2P\exp\left(-\frac{Nc^2}{32}\right).
  \end{align*}

  Additionally, we can union bound to bound the probability that either fails as
  \begin{align*}
    2P\exp\left(-\frac{Nc^2\gamma^2}{8\beta^4}\right)
    +
    2P\exp\left(-\frac{Nc^2}{32}\right)
  \end{align*}

  We require that this be less than $\delta$, so
  \begin{gather*}
    2P\exp\left(-\frac{Nc^2\gamma^2}{8\beta^4}\right)
    +
    2P\exp\left(-\frac{Nc^2}{32}\right)
    \leq \delta \\
    \exp\left(-\frac{Nc^2\gamma^2}{8\beta^4}\right)
    +
    \exp\left(-\frac{Nc^2}{32}\right)
    \leq \frac{\delta}{2P} \\
    \max\left[
    \exp\left(-\frac{Nc^2\gamma^2}{8\beta^4}\right),
    \exp\left(-\frac{Nc^2}{32}\right)
    \right]
    \leq \frac{\delta}{4P} \\
    \min\left[
    \frac{Nc^2\gamma^2}{8\beta^4},
    \frac{Nc^2}{32}
    \right]
    \geq \log\left(\frac{4P}{\delta}\right) \\
    N
    \min\left[
    \frac{c^2\gamma^2}{8\beta^4},
    \frac{c^2}{32}
    \right]
    \geq \log\left(\frac{4P}{\delta}\right) \\
    N
    \geq
    \frac{
    \max\left(
      8\beta^4 / \gamma^2,
    32
    \right)
}{c^2}
    \log\left(\frac{4P}{\delta}\right) \\
  \end{gather*}
\end{proof}

\onecolumn
%\newpage
%\twocolumn
\section{Additional Experimental Results}

%Gen  Model Stats Table
\begin{table}[htbp]
\centering
\caption{Generative Model Statistics and Performance}
\label{table:gen-stats}
\begin{tabular}{@{}ccccccc@{}}
\toprule
\multirow{2}{*}{\textbf{Task}} & \multicolumn{3}{c}{\textbf{Statistics}}                               & \multicolumn{2}{c}{\textbf{Score}} & \multirow{2}{*}{\textbf{SL-SP}} \\ \cmidrule(lr){2-4} \cmidrule(lr){4-6}
                                      & \textbf{\#WS} & \textbf{\% S} & \multicolumn{1}{l}{\textbf{K}} & \textbf{SP}         & \textbf{SL}        &                                 \\ \midrule
CDR                                   & 33            & 74.09                & 3                              & 55.42               & 55.92              & 0.5                             \\
DT                                    & 16            & 43.54                & 3                              & 62.21               & 62.92              & 0.71                            \\ \midrule
Twitter                               & 320           & 100                  & 2                              & 60.97               & 61.23              & 0.26                            \\
MS-COCO                               & 3             & 85.41                & 2                              & 64.75               & 65.25              & 0.5                             \\ \bottomrule
\end{tabular}
\end{table}

\subsection{Features Selected by Difference Model}
In the main text, we provided two example datasets where the features are interpretable and thus helped user tweak their weak supervision sources accordingly. This further improves the generative model over the automated improvements that SL provided. In this section, we elaborate on the the features selected for the Twitter and MS-COCO datasets. 

For the Twitter task, running latent semantic analysis (LSA) on the identified features provides a list of closely related words with ``disappointed'', ``unexpected'', and ``messaged you'' among the top 20. We surmise that this represents the tweets related to customer complaints, which are easy to identify for labelers, thus increasing their accuracy. Heuristically labeling tweets with any of the identified words as negative improves performance over automated Socratic learning by around 1 accuracy point of the end model (Table~\ref{table:user-stats}). This suggests that Socratic learning was able to capture this latent subset properly in an automated manner. 

For MS-COCO, we first cluster the images according to the features that Socratic learning identifies and analyze the captions of these images. The top words in the captions are:  ``people'', ``person'', ``riding'', and ``holding''. The first two phrases suggest the lack of a heuristic that captures groups of people in the images while the second two relate to action words that are usually associated with humans, which improves the accuracy of the heuristics that label images as ``Human''. 
Manually accounting for these latent subsets after Socratic learning augments the generative model leads to an improvement by $2.24$ accuracy points (Table~\ref{table:user-stats}).

%Gen Abs
\begin{table}[htbp]
\caption{Generative Model Performance}
\label{table:gen-abs}
\centering
\begin{tabular}{@{}cccc@{}}
\toprule
\textbf{Application} & \textbf{SP} & \textbf{SL} & \textbf{SL+U} \\ \midrule
\textbf{CDR}         & 55.42       & 55.92       & 55.93         \\
\textbf{DT}          & 62.21       & 62.92       & 63.42         \\ \midrule
\textbf{Twitter}     & 60.97       & 61.23       & 62.42         \\
\textbf{MS-COCO}     & 64.75       & 65.25       & 68.63         \\ \bottomrule
\end{tabular}
\end{table}

\subsection{Generative Model Performance}
We provide details regarding end discriminative model performance in the main text. We provide a similar comparison for the generative model in Table~\ref{table:gen-stats}. For each of these tasks, we measured generative model performance by assigning probabilistic labels to the test set for fair comparison, \emph{even though in reality, the generative model would never be used for final classification on the test set}. Across all tasks, the SL generative model improves over labels assigned by the SP model. Note that even though the labels assigned by both the SP and SL models were probabilistic, we converted them to binary labels in order to compare them to the ground truth labels. This method fails to capture the change in the probabilistic labels, thus leading to seemingly small gains in performance in some cases. Therefore, to properly measure the effect of the improved training labels, we chose to report end discriminative model performance in the main text.

%Disc Abs
\begin{table}[htbp]
\centering
\caption{Discriminative Model Performance}
\label{table:disc-abs}
\begin{tabular}{@{}ccccccc@{}}
\toprule
\textbf{Application}          & \textbf{Model} & \textbf{MV} & \textbf{SP} & \textbf{SL} & \textbf{SL+U} & \textbf{FS} \\ \midrule
\multirow{2}{*}{\textbf{CDR}} & LR             & 47.74       & 48.86       & 52.19       & 52.29         & 54.8        \\
                              & LSTM           & 52          & 53.8        & 54.7        & 55.3          & 58.4        \\
\multirow{2}{*}{\textbf{DT}}  & LR             & 77.98       & 78.63       & 79.72       & 80.87         & 86.47       \\
                              & LSTM           & 60.74       & 70.96       & 71.95       & 72.55         & 81.65       \\ \midrule
\textbf{Twitter}              & LR             & 63.41       & 65.04       & 67.48       & 68.29         & -           \\
\textbf{MS-COCO}              & AlexNet        & 57.95       & 59.47       & 61.7        & 63.94         & 65.08       \\ \bottomrule
\end{tabular}
\end{table}

\subsection{Absolute Generative and Discriminative Model Performance}
In Tables~\ref{table:gen-abs} and ~\ref{table:disc-abs}, we provide absolute performance numbers for both models on the test set. 
We note that the test set is normally not available to users while they write and develop the weak supervision heuristics that are part of the generative model. 
For example, for the CDR task, the generative model uses heuristics that directly apply distant supervision, which would normally not be available at test time.
In some cases, the generative model is more accurate than the discriminative models, but this is without accounting for the lower coverage generative models have. 
The advantage of discriminative models is that it can generalize over all data points, not only what is labeled by the heuristics that are part of the generative model.

\subsection{Text Relation Extraction}
While we report F1 scores for the end discriminative model in the main text, we also include the performance in terms of precision and recall for the discriminative models in Table~\ref{table:F1}.  
Note that in the generative model performance, SL improves the precision and not recall of the weak supervision sources, since SL only improves the generative model's estimate of the accuracy of the existing sources and does not change the number of data points they label.

%F1+P+R Table
\begin{table*}[htbp]
    \caption{Mention-level precision, recall and F1 scores for text relation extraction tasks.}
    \label{table:F1}
\centering
\begin{tabular}{@{}ccccccccccc@{}}
\toprule
\textbf{} & \textbf{Model} & \textbf{}      & \textbf{Gen}   & \textbf{}      & \textbf{}      & \textbf{Log Reg} &                &       & \textbf{LSTM} &                \\ \midrule
          & Method         & DP             & SL             & SL+U           & DP             & SL               & SL+U           & DP    & SL            & SL+U           \\ \midrule
          & Recall         & \textbf{42.09} & 41.43          & 41.45          & 42.72          & \textbf{52.96}   & 51.56          & 59.9  & \textbf{64.5} & 62.9           \\
CDR       & Precision      & 81.13          & \textbf{85.98} & 85.98          & \textbf{57.05} & 51.45            & 53.04          & 48.8  & 47.4          & \textbf{49.4}  \\
          & F1             & 55.42          & 55.92          & \textbf{55.93} & 48.86          & 52.19            & \textbf{52.29} & 53.8  & 54.7          & \textbf{55.3}  \\ \midrule
          & Recall         & \textbf{49.14} & 47.51          & 48.13          & 81.61          & 84.49            & \textbf{86.45} & 94.52 & 98.08         & \textbf{100}   \\
DT        & Precision      & 84.78          & \textbf{93.14} & 92.96          & 75.87          & 75.46            & \textbf{75.98} & 56.81 & 56.81         & \textbf{56.92} \\
          & F1             & 62.21          & 62.92          & \textbf{63.42} & 78.63          & 79.72            & \textbf{80.87} & 70.96 & 71.95         & \textbf{72.55} \\ \bottomrule
\end{tabular}
\centering
\end{table*}

\end{document}